%% file: Ojasalgo.tex
\documentclass[11pt]{article}

\usepackage{amsmath,amsbsy,amsfonts,amssymb,amsthm,fullpage,units,url}
\usepackage{times,caption}

\usepackage{units,stackrel}
\usepackage{algorithmic,algorithm}
\usepackage{color,cases}
\usepackage{hyperref}
\usepackage{dsfont}
\usepackage[numbers]{natbib}
\usepackage{array}

\newcolumntype{C}[1]{>{\centering\let\newline\\\arraybackslash\hspace{0pt}}m{#1}}

\newtheorem{lemma}{Lemma}[section]

\newtheorem{theorem}{Theorem}[section]

\theoremstyle{definition}

\newtheorem{definition}{Definition}

\include{notations}

\definecolor{WildStrawberry}{RGB}{255,67,164}

\renewcommand{\cite}{~\citep}

\title{Streaming PCA: Matching Matrix Bernstein and Near-Optimal Finite Sample Guarantees for Oja's Algorithm}
\author{Prateek Jain\footnote{Microsoft Research India. Email: prajain@microsoft.com} \and Chi Jin\footnote{UC Berkeley. Email: chijin@cs.berkeley.edu} \and Sham M. Kakade\footnote{University of Washington. Email: sham@cs.washington.edu} \and Praneeth Netrapalli\footnote{Microsoft Research New England. Email: praneeth@microsoft.com} \and Aaron Sidford\footnote{Microsoft Research New England. Email: asid@microsoft.com}}

\begin{document}
\maketitle

\begin{abstract}
This work provides improved guarantees for
streaming principle component analysis (PCA). Given
$\A_1, \ldots, \A_n\in \R^{d\times d}$ sampled independently from
distributions satisfying $\E{\A_i} = \Cov$ for $\Cov \succeq \mZero$,
this work provides an $O(d)$-space linear-time single-pass streaming algorithm
for estimating the top eigenvector of $\Cov$. The algorithm nearly
matches (and in certain cases improves upon) the accuracy obtained by
the standard batch method that computes top eigenvector of the
empirical covariance $\frac{1}{n} \sum_{i \in [n]} \A_i$ as analyzed
by the matrix Bernstein inequality. Moreover, to achieve constant
accuracy, our algorithm improves upon the best previous known sample
complexities of streaming algorithms by either a multiplicative factor
of $O(d)$ or $1/\mathrm{gap}$ where $\mathrm{gap}$ is the relative
distance between the top two eigenvalues of $\Cov$.

These results are achieved through a novel analysis of the classic Oja's
algorithm, one of the oldest and most popular algorithms for
streaming PCA. In particular, this work shows that simply picking a random initial point
$\w_0$ and applying the update rule
$\w_{i + 1} = \w_i + \eta_i \A_i \w_i$ suffices to accurately estimate
the top eigenvector, with a suitable choice of
$\eta_i$. We believe our result sheds light on how to efficiently
perform streaming PCA both in theory and in practice, and we hope that
our analysis may serve as the basis for analyzing many variants and
extensions of streaming PCA.
\end{abstract}
\newpage
\input{intro}
\input{related}

\input{prelims}

\input{approach}

\input{results}
\input{lemmas}

\input{conclusion}

\section{Acknowledgements}
Sham Kakade acknowledges funding from the Washington Research Foundation for innovation in
Data-intensive Discovery.

\newpage


\end{document}

%% file: notations.tex
\newcommand{\be}{\begin{equation}}
\newcommand{\ee}{\end{equation}}

\newcommand{\vect}[1]{\ensuremath{\mathbf{#1}}}
\newcommand{\mat}[1]{\ensuremath{\mathbf{#1}}}

\newcommand{\twonorm}[1]{\left\| {#1} \right\|_2}

\newcommand{\iprod}[2]{\left\langle #1, #2 \right\rangle}
\newcommand{\frob}[1]{\left\|#1\right\|_F}
\newcommand{\order}[1]{\ensuremath{\mathcal{O}\left(#1\right)}}

\newcommand{\eqdef}{\stackrel{\triangle}{=}}
\newcommand{\defeq}{\stackrel{\triangle}{=}}
\newcommand{\abs}[1]{\left|#1\right|}

\newcommand{\trace}[1]{\mathrm{Tr}\left(#1\right)}

\newcommand{\trans}[1]{{#1}^{\top}}

\newcommand{\E}[1]{\mathbb{E}\left[#1\right]}
\newcommand{\Eop}{\mathbb{E}}
\renewcommand{\P}[1]{\mathbb{P}\left[#1\right]}
\newcommand{\Var}[1]{\mathrm{Var}\left[#1\right]}
\newcommand\R{\mathbb{R}}

\newcommand{\Cov}{\mat{\Sigma}}
\newcommand{\A}{\mat{A}}
\newcommand{\At}[1]{\mat{A}_{#1}}
\newcommand{\AtTrans}[1]{\mat{A}_{{#1}}^{\top}}
\newcommand{\B}{\mat{B}}
\newcommand{\Bt}[1]{\mat{B}_{#1}}
\newcommand{\BtTr}[1]{\mathbf{B}_{#1}^{\top}}
\newcommand{\C}{\mat{C}}
\newcommand{\D}{\mat{D}}
\newcommand{\G}{\mat{G}}

\newcommand{\mZero}{\mat{0}}
\newcommand{\Wts}[2]{\mat{W}_{#1,#2}}
\newcommand{\Gt}[1]{{\mat{G}_{#1}}}
\newcommand{\WtsTr}[2]{\mathbf{W}_{#1,#2}^{\top}}

\newcommand{\Vperp}{\mat{V}_{\perp}}

\newcommand{\eye}{\mat{I}}

\renewcommand\u{\vect{u}}
\newcommand\x{\vect{x}}
\newcommand\e{\vect{e}}

\newcommand\ustarTr{{{\vect{u}^*}^{\top}}}

\newcommand\vtilde{{\widetilde{\vect{v}}}}

\newcommand\vg{\vect{g}}

\newcommand\Vtildeperp{{\widetilde{\mat{V}}}_{\perp}}
\newcommand\VtildeperpTr{{\widetilde{\mat{V}}}_{\perp}^{\top}}
\renewcommand\v{\vect{v}}
\newcommand{\va}{\vect{a}}

\newcommand\w{\vect{w}}
\newcommand{\wt}[1]{\vect{w_{#1}}}

\newcommand{\dist}{\mathcal{D}}

\newcommand{\variance}{\mathcal{V}}
\newcommand{\variancetotal}{\overline{\mathcal{V}}}
\newcommand{\upperbound}{\mathcal{M}}

%% file: intro.tex
\section{Introduction}\label{sec:intro}

Principal component analysis (PCA) is one of the most fundamental problems in machine learning, numerical linear algebra, and data analysis. It is commonly used for data compression, image processing, and visualization  \cite{jolliffe2002principal} etc.

When we desire to perform PCA on large data sets, it may be the case that we cannot afford more than single pass over the data (or worse to even store the data in the first place) \cite{hall1998incremental,weng2003candid,ross2008incremental}. To alleviate this issue, a popular line of research over the past several decades has been to consider streaming algorithms for PCA under the assumption that the data has reasonable statistical properties \cite{krasulina1970method,oja1982simplified,BalsubramaniDF13,MitgliakisCJ13,SaOR15}. There have been significant breakthroughs in getting near-optimal streaming PCA algorithms under fairly specialized models, e.g. spiked covariance \cite{SaOR15}. 

This work considers one of the most natural variants of PCA,
estimating the top eigenvector of a symmetric matrix, under a mild
(and standard) set
of assumptions under which concentration of measure applies (under the
matrix Bernstein
inequality\cite{vershynin2010introduction,Tropp12}). In particular,
the setting is as follows: 

\begin{definition}[Streaming PCA]
\label{def:streaming-PCA}
Let $\A_1, \A_2, ... , \A_n \in \R^{d \times d}$ be a sequence  of (not necessarily symmetric) matrices sampled independently from distributions that satisfy the following: \begin{enumerate}
	\item	$\E{\At{i}} = \Cov$ for symmetric positive semidefinite (PSD) matrix $\Cov \in \R^{d \times d}$,
	\item	$\twonorm{\At{i} - \Cov} \leq \upperbound$ with probability $1$, and
	\item	$\max\left\{\twonorm{\E{(\At{i}-\Cov)\trans{(\At{i} - \Cov)}}}, \twonorm{\E{\trans{(\At{i}-\Cov)}(\At{i}-\Cov)}}\right\} \leq \variance$.
\end{enumerate}
Let $\v_1, ..., \v_d$ denote the eigenvectors of $\Cov$ and $\lambda_1 \geq ... \geq \lambda_d$ denote the corresponding eigenvalues. Our goal is to compute an $\epsilon$-approximation to $\v_1$, that is a unit vector $\w$ such that
$
\sin^2(\w, \v_1) \eqdef 1 - (\w^\top \v_1)^2 \leq \epsilon
$, in a single pass while minimizing space, time, and error (i.e. $\epsilon$). Note that $\sin(\w, \v_1)$ denotes the $\sin$ of the angle between $\w$ and $\v_1$.
\end{definition}

A special case of Streaming PCA is to estimate the top eigenvector of the covariance matrix of a distribution $\dist$ over $\R^{d}$, i.e.  given independent samples $\va_1, ..., \va_n \in \R^d$ estimate the top eigenvector of $\Eop_{\va \sim \dist}[\va \va^\top]$. This encompasses the popular "spiked covariance model"\cite{johnstone2001distribution}.

It is well known that to solve the Streaming PCA problem, one can simply compute the empirical covariance matrix $\frac{1}{n} \sum_{i \in [n]} \A_i$ and compute the right singular vector of this matrix. Here, matrix Bernstein inequality\cite{vershynin2010introduction,Tropp12} and Wedin's theorem\cite{wedin1972perturbation} implies the following standard sample complexity bound for the Streaming PCA problem:

\begin{theorem} [Eigenvector Concentration using matrix Bernstein and Wedin's theorem]
\label{thm:bernstein_kahan}
Under the assumptions of Definition~\ref{def:streaming-PCA}, 
the top right singular vector $\widehat{\v}$ of $\widehat{\Cov}=\frac{1}{n} \sum_{i \in [n]} \A_i$ is an $\epsilon$-approximation to the top eigenvector $\v_1$ of $\Cov$ with probability $1-\delta$, where  
\begin{equation*}
\sin^2(\widehat{\v}, \v_1) \leq \epsilon
\leq \frac{16 \variance \log \frac{d}{\delta}}{(\lambda_1 -\lambda_2)^2} \cdot \frac{1}{n} + \left(\frac{4\upperbound \log\frac{d}{\delta}}{\lambda_1 - \lambda_2 }\right)^2 \cdot \frac{1}{n^2}.
\end{equation*}
\end{theorem}

Theorem~\ref{thm:bernstein_kahan} is essentially the previous best
sample complexity known for estimating the top eigenvector~\footnote{In recent work in\cite{jin2015robust} it was shown
  that the $log(d/\delta)$ factor in the first term could be removed
  asymptotically for small enough $\epsilon$ if only constant success
  probability is required.}. 
Unfortunately, the above is purely a statistical claim, and,
algorithmically, there are least two concerns. First, computing the
empirical covariance matrix
$\widehat{\Cov}=\frac{1}{n}\sum_{i \in [n]} \A_i$ naively requires
$O(d^2)$ time and space, and second, computing the top eigenvector of
the empirical covariance matrix in general may require super linear
time\cite{golub2012matrix}. While there have been many attempts to
produce streaming algorithms that use only $O(d)$ space to solve the
streaming PCA problem, to our knowledge, all previous methods
either lose a multiplicative factor of either
$\frac{\lambda_1}{\lambda_1 - \lambda_2}$ or $d$ in the analysis in
order to achieve constant accuracy when applied in our
setting\cite{BalsubramaniDF13,MitgliakisCJ13, HardtP14,
  SaOR15,jin2015robust}.

In an attempt to overcome this limitation and improve the guarantees
for solving the streaming PCA problem, this work seeks to address the following question:

\begin{quote}
\emph{
Can we match the sample complexity of matrix Bernstein + Wedin's theorem with an algorithm that uses $O(d)$ space only and takes a single linear-time pass over the input?}
\end{quote}

This work answers this question in the affirmative, showing that one
can succeed with constant probability matching the sample complexity
of Theorem~\ref{thm:bernstein_kahan} up to logarithmic terms and small
additive factors. Interestingly, this is achieved by providing a novel
analysis of the classical Oja's algorithm, which is perhaps, the most
popular algorithm for Streaming PCA\cite{oja1982simplified}.

\begin{algorithm}[H]
	\caption{Oja's algorithm for computing top eigenvector}
	\begin{algorithmic}\label{algo:oja}
		\renewcommand{\algorithmicrequire}{\textbf{Input: }}
		\renewcommand{\algorithmicensure}{\textbf{Output: }}
		\REQUIRE $\At{1},\cdots,\At{n}$.
		\STATE Choose $\wt{0}$ uniformly at random from the unit sphere
		\FOR{$t = 1,\cdots,n$}
		\STATE $\wt{i} \leftarrow \wt{i-1} + \eta_i \At{i}\wt{i-1}$
		\STATE $\wt{i} \leftarrow \wt{i}/\twonorm{\wt{i}}$
		\ENDFOR
		\ENSURE  $\wt{n}$
	\end{algorithmic}
\end{algorithm}

Oja's algorithm is one of the simplest algorithms one would imagine
for the streaming PCA problem (See Algorithm~\ref{algo:oja}). In fact,
due to its simplicity, it was proposed a neurally plausible algorithm.
In the case that each $\A_i$ comes from the same distribution $\dist$
it corresponds to simply performing projected stochastic gradient
descent on the objective function of maximizing the Rayleigh Quotient
over the distribution
$\max_{\twonorm{w} = 1} \Eop_{\A \sim \dist} \w^\top \A \w$. It is
well known that under very mild conditions on the stepsize sequence,
Oja's algorithm asymptotically converges to the top eigenvector of the
covariance matrix $\Cov$\cite{oja1982simplified}. However, obtaining
optimal rates of convergence, let alone finite sample guarantees, for
Streaming PCA has been quite challenging. The best known results are
off from Theorem~\ref{thm:bernstein_kahan} by a factor of
$\order{d}$\cite{SaOR15}.

This work shows that for proper choice of learning rates $\eta_i$,
Oja's algorithm in fact can improve the best known results for
streaming PCA and answer our question in the affirmative. In
particular, we have that:

\begin{theorem}\label{thm:intro1}
Let the assumptions of Definition~\ref{def:streaming-PCA} hold. Suppose the step size sequence for Algorithm~\ref{algo:oja} is chosen to be $\eta_i = \frac{\log d}{(\lambda_1-\lambda_2)(\beta+i)}$, where
	\begin{align*}
	\beta \eqdef 40\max\left(\frac{ \upperbound \log d}{(\lambda_1-\lambda_2)},\frac{\left(\variance+\left(\lambda_{1}\right)^2\right) \log^2 d}{(\lambda_1-\lambda_2)^2}\right).
	\end{align*}
	Then the output $\wt{n}$ of Algorithm~\ref{algo:oja} is an $\epsilon$-approximation to the top eigenvector $\v_1$ of $\Cov$ satisfying
	\begin{align*}
	\sin^2(\w_n,\v_1)\leq \epsilon \;\leq \;C \left(\frac{\variance \log d}{(\lambda_1-\lambda_2)^2}\cdot \frac{1}{n} + \left(\frac{2\beta}{n}\right)^{2 \log d}\right),
	\end{align*}
	with probability greater than $3/4$. Here $C$ is an absolute numerical constant.
\end{theorem}
The error above should be interpreted as being the sum of a
$\order{\frac{1}{n}}$ higher order term and another
$\order{(2\beta/n)^{2 \log d}}$ lower order term which is at most
$o(\frac{1}{n^{\log d}})$ (once $n > 4\beta^2$). In particular, this
result shows that, up to an additive lower order term, one can match
Theorem~\ref{thm:bernstein_kahan} with an asymptotic error of
$\order{\frac{\variance \log d}{(\lambda_1-\lambda_2)^2 n}}$ with
constant probability. The lower order term has $\beta$ which is the
$\max$ of three parts:
$\frac{ \upperbound \log d}{(\lambda_1-\lambda_2)}$,
$\frac{ \variance \log^2 d}{(\lambda_1-\lambda_2)^2}$ and
$\frac{ \lambda_1^2 \log^2 d}{(\lambda_1-\lambda_2)^2}$. The first
part, depending on $\upperbound$, is exactly the same as what appears
in Theorem~\ref{thm:bernstein_kahan}. The second one, depending on
$\variance$ has an additional $\log d$ factor over the first order
term and is irrelevant once, say $n > 10 \beta$. Notably, the third part,
depending on $\lambda_1^2$, does not appear in
Theorem~\ref{thm:bernstein_kahan}; it arises here entirely due to
computational reasons: the setting allows only a \emph{single linear-time
  pass} over the matrices, while Theorem~\ref{thm:bernstein_kahan}
makes no such assumption. For instance, consider the case
$\variance=0$ which means $\At{1}=\Cov$. Matrix Bernstein tells us
that one sample is sufficient to compute $\v_1$. However, it is not
evident how to compute it using a single pass over $\At{1}$. Note
however, that the rate at which the lower order terms,
i.e. $o\left(\frac{1}{n^{\log d}}\right)$, decrease is much better than
$\order{1/n^2}$ guaranteed by Theorem~\ref{thm:bernstein_kahan}.

In fact, this result also improves the asymptotic error rate obtained
by Theorem~\ref{thm:bernstein_kahan}. In particular, the
following result shows that Oja's algorithm gets an asymptotic rate of
$\order{\frac{\variance }{(\lambda_1-\lambda_2)^2 n}}$ which is better
than that of matrix Bernstein by a factor of
$\order{\log d}$.\footnote{A similar asymptotic result was recently
  obtained by\cite{jin2015robust}. However, their result requires an
  initial vector that is constant close to $\v_1$, which itself is a
  difficult problem.} \begin{theorem}\label{thm:intro2} Let the
  assumptions of Definition~\ref{def:streaming-PCA} hold. Suppose the
  step size sequence for Algorithm~\ref{algo:oja} is chosen to be
  $\eta_i = \frac{6}{(\lambda_1-\lambda_2)(\beta+i)}$, where
  \begin{align*} \beta \eqdef 720\max\left(\frac{ \upperbound
      }{(\lambda_1-\lambda_2)},\frac{\variance + \lambda_{1}^2
      }{(\lambda_1-\lambda_2)^2}\right). \end{align*} Suppose
  $n > \beta^{1.2}d^{0.1}$. Then the output $\wt{n}$ of
  Algorithm~\ref{algo:oja} is an $\epsilon$-approximation to the top
  eigenvector $\v_1$ of $\Cov$ satisfying \begin{align*}
    \sin^2(\w_n,\v_1)\leq \epsilon \;\leq \; C \left(\frac{ \variance
      }{(\lambda_1-\lambda_2)^2}\cdot \frac{1}{n} +
      \frac{1}{n^2}\right), \end{align*} with probability greater than
  $3/4$. Here $C$ is an absolute numerical constant. \end{theorem}

Note that Theorems~\ref{thm:intro1} and~\ref{thm:intro2} guarantee
success probability of $3/4$. One way to boost the probability to
$1-\delta$, for some $\delta > 0$, is to run $\order{\log 1/\delta}$
copies of the algorithm, each with $3/4$ success probability and then
output the geometric median of the solutions, which can be done in
nearly linear time\cite{cohen2016median}. The detailes are omitted
here.

Beyond the improved sample complexities we believe our analysis sheds
light on the type of step sizes for which Oja's algorithm converges
quickly and therefore illuminates how to efficiently perform streaming
PCA. We note that we have essentially assumed an oracle which sets the
step size sequence, and an important question is how to set the step
size in a robust and data data driven manner. Moreover, we believe that our
analysis is fairly general and hope that it may be extended to make
progress on analyzing the many variants of PCA that occur in both
theory and in practice.


%% file: related.tex
\subsection{Comparison with Existing Results}\label{sec:comp}
Here we compare our sample complexity bounds with existing analyses of various methods. Recall that the error of the estimate $\w$ is $\sin^2(\w, \v_1)=1-(\w^\top \v_1)^2$.

We consider three popular methods used for computing $\v_1$. The first one is the batch method which computes largest eigenvector of empirical covariance and uses Wedin's theorem with matrix Bernstein inequality (cf. Theorem~\ref{thm:bernstein_kahan}). The second method is Alecton, which is very similar to Oja's algorithm \cite{SaOR15}. Finally, consider a block-power method (BPM) \cite{HardtP14, MitgliakisCJ13} which divides samples into different blocks and applies power iteration to the empirical estimate from each block. See Table~\ref{tab:comp} for the comparison.

We stress that some of the results we compare to make different assumptions than Definition~\ref{def:streaming-PCA}. The bounds stated for them are our best attempt to adapt their bounds in the setting of Definition~\ref{def:streaming-PCA} (which is quite standard). The next paragraph provides a simple example, which demonstrates the improvement in our result as compared to existing work.
%

Let $\At{i}=\x_i \x_i^{\top}$, where $\x_i\in \R^d$ and $\x_i=\e_1$ with probability $1/d$ and $\x_i=\sigma \e_j, 1\leq j\leq d$ with probability $1/d$ where $\e_j$ denotes the $j^{\textrm{th}}$ standard basis vector and $\sigma<1$. Note that $\Cov=\E{\At{i}}=\frac{(1-\sigma^2)}{d}\e_1\e_1^T+\frac{1}{d}\sigma^2 \eye$, $\|\At{i}\|_2\leq 1$ for all $i$, and $\|\E{\At{i}\At{i}^{\top}}\|_2\leq \frac{1}{d}$. Even for constant accuracy $\epsilon = \Omega(1)$, Theorem~\ref{thm:intro1} tells us that $n = \order{\frac{d \log^2 d}{(1-\sigma^2)^2}}$ is sufficient. On the other hand, Theorem $1$ of \cite{SaOR15} requires $n = \order{\frac{d^2 \log^2 d}{(1-\sigma^2)^2}}$, while Theorem $2.4$ of \cite{HardtP14} requires $n = \order{\frac{d \log^2 d}{(1-\sigma^2)^3}}$. Asymptotically, as $n$ becomes larger, our error scales as $\order{\frac{d}{(1-\sigma^2)^2 }\cdot \frac{1}{n}}$ while that of~\cite{SaOR15} scales as $\order{\frac{d^2}{(1-\sigma^2)^2 }\cdot \frac{\log n}{n}}$ and that of~\cite{HardtP14} scales as $\order{\frac{d}{(1-\sigma^2)^3 }\cdot \frac{\log n}{n}}$. Combining matrix Bernstein and Wedin's theorems gives an asymptotic error of $\order{\frac{d \log d}{(1-\sigma^2)^2 }\cdot \frac{1}{n}}$.
%
\begin{table}
\centering
\begin{tabular}[h]{|C{5cm}|c|c|}
\hline
Algorithm & Error & $\order{d}$ space? \\
\hline
Oja's \textbf{(this work, Theorem~\ref{thm:main})} & $\order{\frac{\variance }{(\lambda_1-\lambda_2)^2} \cdot \frac{1}{n}}$ & Yes \\
\hline
Matrix Bernstein + Wedin's theorem (Theorem~\ref{thm:bernstein_kahan}) & $\order{\frac{\variance\log d}{(\lambda_1-\lambda_2)^2}\cdot \frac{1}{n}}$ 
& No \\
\hline
Alecton \cite{SaOR15} & $\order{\frac{\variance d}{(\lambda_1-\lambda_2)^2}\cdot \frac{\log n}{n}}$
& Yes \\
\hline
Block Power Method \cite{HardtP14} & $\order{\frac{\variance \lambda_1 \log d}{(\lambda_1-\lambda_2)^3} \cdot \frac{\log n}{n}}$
& Yes \\
\hline
\end{tabular}
\caption{Asymptotic error guaranteed by various methods under assumptions of Definition~\ref{def:streaming-PCA} with at least constant probability, and ignoring constant factors. Recall that the error is defined as $\sin^2(\w, \v_1) = 1-(\w^\top \v_1)^2$. Our analysis provides the optimal $1/n$ error decay rate as compared to Alecton and Block power method which obtain $\frac{\log n}{n}$. Moreover, our bound is $O(d)$ tighter than that of Alecton \cite{SaOR15} and $O(\frac{\lambda_1}{\lambda_1-\lambda_2})$ tighter bound than that of Block Power Method \cite{HardtP14}. The assumptions made in \cite{SaOR15} for Alecton are different from our (more standard) assumption; we have optimized their bounds are optimized in our setting. See Section~\ref{sec:comp} for a concrete example where our analysis provides these improvements over \cite{SaOR15,HardtP14}.}\label{tab:comp}
\end{table}

\subsection{Additional Related Work}
Existing results for computing largest eigenvector of a data covariance matrix using streaming samples can be divided into three broad settings: a) stochastic data, b) arbitrary sequence of data, c) regret bounds for arbitrary sequence of data.

{\bf Stochastic data}: Here, the data is assumed to be sampled i.i.d. from a fixed distribution. The analysis of Oja's algorithm as well as those of block power method and Alecton mentioned earlier are in this setting. \cite{MitgliakisCJ13} also obtained a result in the restricted spiked covariance model.
\cite{BalsubramaniDF13}  provides an analysis of a modification of Oja's algorithm but with an extra $O(d^5)$ multiplicative factor compared to ours.~\cite{jin2015robust} provides an algorithm based on shift and invert framework that obtains the same asymptotic error as ours. However, their algorithm requires warm start with a vector that is already constant close to the top eigenvector, which itself is a hard problem.

{\bf Arbitrary data}: In this setting, each row of the data matrix is provided in an arbitrary order. Most of the existing methods here first compute a sketch of the matrix and use that to compute an estimate of the top eigenvector \cite{ClarksonW09, Liberty13, nelson2013osnap, cohen2015optimal, GhashamiLPW15, BoutsidisGKL15}. However, a direct application of such techniques to the stochastic setting leads to sample complexity bounds which are larger by a multiplicative factor of $O(d)$ (ignoring other factors like variance etc). Finally, \cite{Shamir15, garber2015fast,jin2015robust} also provide methods for eigenvector computation, but they require multiple passes over the data and hence do not apply to the streaming setting.

{\bf Regret bounds}: Here, at each step the algorithm has to output an estimate $\w$ of $\v_1$ for which we get reward of $\w^T \A_i \w$ and the goal is to minimize the regret w.r.t. $\v_1$. The algorithms in this regime are mostly based on online convex optimization and applying them in our setting would again result in a loss of multiplicative $O(d)$. Moreover, typical algorithms in this setting are not memory efficient \cite{WarmuthK06, GarberHM15}.


%% file: prelims.tex
\subsection{Notation}\label{sec:notation}
Bold lowercase letters such as $\u,\v,\w$ are used to denote vectors
and bold uppercase letters such as $\A,\B,\C$ to denote matrices. For
symmetric matrices $\A$ and $\B$, $\A \preceq \B$  denotes the
condition that $\x^\top \A \x \leq \x^\top \B \x$ for all $\x$ and
define $\B \succeq \A$ analogously. A symmetric matrix $\A$
is positive semidefinite if $\A \succeq \mZero$. For symmetric matrices
$\A, \B$, define their inner product as
$\iprod{\A}{\B} \eqdef \trace{\A^\top \B}$.

\subsection{Paper Organization}
The rest of this paper is organized as
follows. Section~\ref{sec:preliminaries} introduces  basic
mathematical facts used throughout the paper and also provides a proof
of the error bound of the standard batch method
(Theorem~\ref{thm:bernstein_kahan}). Section~\ref{sec:approach}
provides an overview of our approach to analyzing Oja's algorithm and
provides the main technical result of the paper. This technical result
is used in Section~\ref{sec:results} to prove the running time for
Oja's algorithm and to justify the choice of step size.
Section~\ref{sec:proofs} presents the proof of the main technical
result. Section~\ref{sec:conc} concludes  and mentions a few interesting future directions. 

\section{Preliminaries}
\label{sec:preliminaries}

The following basic inequalities regarding power
series, the exponential, and PSD matrices are used throughout. The
facts are summarized here:

\begin{lemma}[Basic Inequalities]\label{lem:trace-cauchy-schwarz}
The following are true:
\begin{itemize}
	\item	$1+x \leq \exp(x)$ for all $x$
	\item	$1+x \geq \exp\left(x-x^2\right)$ for all $x \geq 0$
	\item	$\frac{1}{1 + x} \leq \sum_{i=1}^{\infty} \frac{1}{(x+i)^2} \leq \frac{1}{x}$
	\item   $\iprod{\A}{\B} \leq \iprod{\A}{\C}$ for PSD matrices $\A, \B, \C$ with $\B \preceq \C$ 
	\item	$\trace{\A^\top \B} \leq \frac{1}{2} \trace{\A^\top \A + \B^\top \B}$ for all matrices $\A,\B \in \R^{m \times n}$.
\end{itemize}

\begin{proof}
The first inequality follows from the Taylor expansion of $\exp(x)$. The second comes from $1 + 0 = \exp(0 - 0^2)$ and $\frac{d}{dx} (1 + x) \leq \frac{d}{dx} \exp(x - x^2)$ for $x \geq 0$. The third follows by considering upper and lower Riemann sums of $\int_{y=1}^{\infty} 1/(x+y)$.
The fourth from the fact that since $\A$ is PSD there is a matrix $\D$ with $\D^\top \D = \A$ and therefore 
\[
\iprod{\A}{\B} = \trace{\A^\top \B} = \trace{\D \B \D^\top}
\leq \trace{\D \C \D^\top} = \iprod{\A}{\C} ~.
\]
The final follows from Cauchy Schwarz and Young's inequality, i.e. $x\cdot y \leq \frac{1}{2}(x^2 + y^2)$ as
\[
\trace{\B^\top \A}
= \sum_{i \in [n]} 1_i \B^\top \A 1_i
\leq \sum_{i \in [n]} \twonorm{\A 1_i} \cdot \twonorm{\B 1_i} 
\leq \frac{1}{2} \sum_{i \in [n]} \left(\twonorm{\A 1_i}^2 + \twonorm{\B 1_i}^2\right)
\]
\end{proof}

\end{lemma}

The following is a matrix Bernstein based proof of the error bound of the batch method. 
\begin{proof}[Proof of Theorem~\ref{thm:bernstein_kahan}]
Using Theorem 1.4 of \cite{Tropp12}, we have (w.p. $\geq 1-\delta$): 
\begin{equation}
\twonorm{\frac{1}{n}\sum_{i=1}^n \At{i} - \Cov } \le 2 \cdot \max\left\{\sqrt{\frac{\variance}{n}\log \frac{d}{\delta}}, \frac{\upperbound}{n}\log \frac{d}{\delta}\right\}.\label{eq:mb_1}
\end{equation}
Let $\widehat{\v}$ be the top eigenvector of $\widehat{\Cov}=\frac{1}{n}\sum_{i=1}^n \At{i}$. Using Wedin's theorem \cite{wedin1972perturbation}, implies:
\begin{equation}
\sin^2\langle\v_1, \widehat{\v}\rangle
\le \frac{\twonorm{\frac{1}{n}\sum_{i=1}^n \At{i} - \Cov}^2}{|\lambda_1-\lambda_2|^2 }. \label{eq:mb_2}
\end{equation}
Theorem now follows by combining \eqref{eq:mb_1} and \eqref{eq:mb_2}. \end{proof}


%% file: approach.tex
\section{Approach}\label{sec:approach}

Let us now describe the approach to analyze Oja's algorithm. 
We provide our main theorem regarding the convergence rate of Oja's algorithm and discuss how it is proved. 
The details of the proof are deferred to Section~\ref{sec:proofs} and the use of the theorem to choose step sizes is in Section~\ref{sec:results}.

One of the primary difficulties in analyzing Oja's algorithm, or more broadly any algorithm for streaming PCA, is choosing a subtle potential function to analyze the  method. If we try to analyze the progress of Oja's algorithm in every iteration $i$, by measuring the quality of $\w_i$, we run the risk that during the first few iterations of Oja's algorithm a step may actually yield a $\w_{i + 1}$ that is orthogonal to $\v_{i}$. If this happens, even in the typical best case, where all future samples are $\Cov$ itself, we would still fail to converge. In short, if we do not account for the randomness of $\w_0$ in our potential function then it is difficult to show that a rapidly convergent algorithm does not catastrophically fail.

Rather than analyzing the convergence of $\w_i$ directly we instead analyze the convergence of Oja's algorithm as an operator on $\w_0$.  Oja's algorithm simply considers the matrix 
\begin{equation}
\label{eq:onestep-1}
\B_{n} \defeq (\eye + \eta_n \A_n )(\eye + \eta_{n - 1} \A_{n - 1}) \cdots (\eye + \eta_1 \A_1)
\end{equation}
and outputs the normalized result of applying this matrix, $\B_{n}$, to the random initial vector, i.e.
\begin{equation}
\label{eq:onestep-2}
\w_n  = \frac{\B_n \w_0}{\| \B_n \w_0 \|_2} ~.
\end{equation}
Rather than analyze the improvement of $\w_{n + 1}$ over $\w_n$ we analyze $\B_{n + 1}$'s improvement over $\B_n$.

Another interpretation of \eqref{eq:onestep-1} and \eqref{eq:onestep-2} is that Oja's algorithm simply approximates $\v_n$ by performing 1 step of the power method on the matrix $\B_{n}$. Fortunately, analyzing when 1 step of the power method succeeds is fairly straightforward as we show below:

\begin{lemma}[One Step Power Method]\label{lem:onesteppowermethod}
Let $\B \in \R^{d \times d}$, let $\vtilde \in \R^{d}$ be a unit vector, and let $\Vtildeperp$ be a matrix whose columns form an orthonormal basis of the subspace orthogonal to $\vtilde$. If $\w \in \R^{d}$ is chosen uniformly at random from the surface of the unit sphere then with probability at least $1-\delta$
	\begin{align*}
	\sin^2\left(\vtilde, \frac{\B\w}{\|\B\w\|_2}\right)=1-\left( \frac{\vtilde^\top\B \w}{\twonorm{\B \w}}\right)^2\leq \frac{C\log \left(1/\delta\right)}{\delta}\frac{\trace{\VtildeperpTr\B\trans{\B}\Vtildeperp}}{\trans{\vtilde}\B\trans{\B}\vtilde}
	\end{align*}
where $C$ is an absolute constant.
\end{lemma}

\begin{proof}
As $\w$ is distributed uniformly over the sphere, we have: $\w=\vg/\|\vg\|_2$ where $\vg\sim N(0, I)$. Consequently, with probability at least $1-\delta$
\begin{align*}
1-\left(\frac{\vtilde^\top\B \w}{\twonorm{\B \w}}\right)^2&=\frac{\vg^\top \B^\top (\eye-\vtilde\vtilde^\top)\B\vg}{\vg^\top \B^\top \B \vg}\stackrel{\zeta_1}{\leq}\frac{C_1}{\delta} \frac{\vg^\top \B^\top (\eye-\vtilde\vtilde^\top)\B\vg}{\vtilde^\top\B\B^\top\vtilde}\\
&\stackrel{\zeta_2} {\leq}\frac{C\log(1/\delta)}{\delta} \frac{\trace{\B^\top (\eye-\vtilde\vtilde^\top)\B}}{\vtilde^\top\B\B^\top\vtilde},
\end{align*}
where $C_1$ and $C$ are absolute constants. $\zeta_1$ follows as $\vg^\top\B^\top\B\vg\geq (\vtilde^\top\B\vg)^2\geq \frac{1}{C_1}\vtilde^\top\B\B^\top\vtilde$ where the second inequality follows from the fact that $\vtilde^\top\B\vg$ is a Gaussian random variable with variance $\twonorm{\B^\top \vtilde}^2$. Similarly, $\zeta_2$ follows from the fact that$\vg^\top \B^\top (\eye-\vtilde\vtilde^\top)\B\vg$ is a $\chi^2$ random variable with $\trace{\B^\top (\eye-\vtilde\vtilde^\top)\B}$-degrees of freedom. 
\end{proof}

This lemma makes our goal clear. To show that Oja's algorithm succeeds we simply need to show that with constant probability $\v_1^\top \B_n \B_n^\top \v_1$ is relatively large and $\trace{\Vperp \B_n \B_n^\top \Vperp}$ is relatively small, where $\Vperp$ is a matrix whose columns form an orthonormal basis of the subspace orthogonal to $\v_1$. This immediately alleviates the issues of catastrophic failure that plagued analyzing $\w_n$. So long as we pick $\eta_i$ sufficiently small, i.e. $\eta_i = O(1/\max\{\upperbound,\lambda_1\})$ then $\eye + \eta_i \A_i$ is invertible. In this case $\B_n \B_n^\top$ is invertible and $\v_1^\top \B_n \B_n^\top \v_1 > 0$. In short, so long as we pick $\eta_i$ sufficiently small the quantity we wish to bound $\trace{\Vperp \B_n \B_n^\top \Vperp} / \v_1^\top \B_n \B_n^\top \v_1$ is always finite.

To actually bound $\v_1^\top \B_n \B_n^\top \v_1$ and
$\trace{\Vperp^{\top} \B_n \B_n^\top \Vperp}$ we split the analysis
into several parts in Section~\ref{sec:proofs}. First, we show that
$\E{\trace{\Vperp^{\top} \B_n \B_n^\top \Vperp}}$ is small, which
implies by Markov's inequality that $\trace{\Vperp^{\top} \B_n
  \B_n^\top \Vperp}$ is small with constant probability. Then, we show
that $\Eop \v_1^\top \B_n \B_n^\top \v_1$ is large and that $\Var{
  \v_1^\top \B_n \B_n^\top \v_1}$ is small. By Chebyshev's inequality
this implies that ${\v_1^\top \B_n \B_n^\top \v_1}$ is large with
constant probability. Putting these together we achieve the main
technical result regarding the analysis of Oja's method. Once we
devise this roadmap, the proof is fairly straightforward.  

\begin{theorem}[Oja's Algorithm Convergence Rate] \label{thm:main1}
Let $\delta > 0$ and step sizes $\eta_i \leq \frac{1}{4\cdot \max\{M,\lambda_1\}}$. 	The output $\wt{n}$ of Algorithm~\ref{algo:oja} is an $\epsilon$-approximation to $\v_1$ with probability at least  $1 - \delta$ where
\begin{align*}
\epsilon &\leq \frac{1}{Q} \exp\left(5 \overline{\variance}\sum_{i\in[n]}\eta_i^2\right)
\left(d \cdot \exp\left(-2(\lambda_1-\lambda_2)\sum_{i \in [n]} \eta_i  \right) + \variance \sum_{i=1}^{n} \eta_i^2 \exp
\left(-\sum_{j =i+1 }^n 2\eta_j(\lambda_1-\lambda_2) \right)\right)
\end{align*}
 where $Q \eqdef \frac{\delta^2}{C\log(1/\delta)}\left( 1 - \frac{1}{\sqrt{\delta}}\sqrt{ \exp\left(18\overline{\variance}\sum_{i=1}^{n} \eta_i^2 \right) -  1}\right)$, $\overline{\variance} \eqdef \variance+\lambda_1^2$, and $C$ is an absolute constant.
\end{theorem}

Theorem~\ref{thm:main1} is proved in Section~\ref{sec:proofs}. Theorem~\ref{thm:main1} serves as the basis for our results regarding Oja's algorithm. In the next section we show how to use this theorem to choose step sizes and achieve the main results of this paper. 

%% file: results.tex
\section{Main Results}\label{sec:results}

Theorem~\ref{thm:main1}, from the previous section, leads to our main results, provided here. The theorem and proof are below and essentially consist of choosing appropriate parameters to efficiently apply Theorem~\ref{thm:main1}. Once we have this theorem, Theorems~\ref{thm:intro1} and~\ref{thm:intro2} follow by choosing $\alpha = \log d$ and $\alpha = 6$ respectively.

\begin{theorem}\label{thm:main}
	Fix any $\delta > 0$ and suppose the step sizes are set to 
	$\eta_t = \frac{\alpha}{(\lambda_1-\lambda_2)(\beta+t)}$ for $\alpha > \frac{1}{2}$ and
	\begin{align*}
	\beta \eqdef 20\max\left(\frac{ \upperbound \alpha}{(\lambda_1-\lambda_2)},\frac{\left(\variance+\left(\lambda_{1}\right)^2\right) \alpha^2}{(\lambda_1-\lambda_2)^2\log \left(1+\frac{\delta}{100}\right)}\right).
	\end{align*}
	Suppose the number of samples $n > \beta$.
	Then the output $\wt{n}$ of Algorithm~\ref{algo:oja} satisfies:
	\begin{align*}
	1-(\wt{n}^\top \v_1)^2 \leq \frac{C\log(1/\delta)}{\delta^2}\left({d} \left(\frac{\beta}{n}\right)^{2 \alpha} + \frac{\alpha^2 \variance}{(2\alpha-1)(\lambda_1-\lambda_2)^2}\cdot \frac{1}{n}\right),
	\end{align*}
	with probability at least $1-\delta$. Here $C$ is an absolute numerical constant.
\end{theorem}

\input{proof_thm}

%% file: proof_thm.tex
\begin{proof}
	Recall that Theorem~\ref{thm:main1} gives a bound of
	\begin{align}\label{eqn:master}
		\frac{1}{Q} \exp\left(5 \overline{\variance}\sum_{i\in[n]}\eta_i^2\right)
		\left(d \cdot \exp\left(-2(\lambda_1-\lambda_2)\sum_{i \in [n]} \eta_i  \right) + \variance \sum_{i=1}^{n} \eta_i^2 \exp
		\left(-\sum_{j =i+1 }^n 2\eta_j(\lambda_1-\lambda_2) \right)\right)
	\end{align}
	where $Q \eqdef \frac{\delta^2}{C\log(1/\delta)}\left( 1 - \frac{1}{\sqrt{\delta}}\sqrt{ \exp\left(18\overline{\variance}\sum_{i=1}^{n} \eta_i^2 \right) -  1}\right)$. 
	Since $\eta_i=\frac{\alpha}{(\lambda_1-\lambda_2)(\beta + i)}$, we have $\sum_{i\in [n]} \eta_i^2 \leq \frac{\alpha^2}{(\lambda_1-\lambda_2)^2\beta}$ and by our assumption that $\frac{\overline{\variance} \alpha^2}{(\lambda_1-\lambda_2)^2\beta} \leq \frac{1}{18}\log \left(1+\frac{\delta}{100}\right)$, we have: 
	\begin{equation}
	\label{eqn:main4}
	\exp\left(18 \overline{\variance}\sum_{i\in[n]}\eta_i^2\right) \leq \sqrt{2} \quad \Rightarrow \quad 
	Q\geq \frac{\delta^2}{C \log(1/\delta)}.
	\end{equation}
	Moreover, since $\sum_{i \in [n]} \eta_i \geq \frac{\alpha}{\lambda_1-\lambda_2} \log\left(1+n/\beta\right)$, we have
	\begin{align}
		\exp\left(-2(\lambda_1-\lambda_2)\sum_{i \in [n]} \eta_i\right) \leq \left(\frac{\beta}{\beta+n}\right)^{2\alpha}. \label{eqn:main4-5}
	\end{align}
	Note that $\sum_{j=i+1}^n \eta_j\leq \frac{\alpha}{\lambda_1-\lambda_2} \log\frac{n+\beta +1}{i+\beta+1}$. Moreover, as $\alpha > 1/2$, we have: 
	\begin{align}
	  &\sum_{i=1}^{n}\eta_i^2 \exp\left(-2(\lambda_1-\lambda_2)\sum_{j=i+1}^n \eta_j\right)\nonumber\\
	  &\leq \frac{\alpha^2}{(\lambda_1-\lambda_2)^2}\sum_{i=1}^{n} \frac{1}{(\beta+i)^2}\exp\left(2\alpha\log\frac{i+\beta+1}{n+\beta+1}\right),\notag\\&\leq \frac{(\beta+1)^2}{\beta^2}\cdot \frac{\alpha^2}{(\lambda_1-\lambda_2)^2(n+\beta+1)^{2\alpha}}\cdot \sum_{i=1}^n (i+\beta+1)^{2\alpha-2},\notag\\
	&\leq \frac{2\alpha^2}{(2\alpha-1)(\lambda_1-\lambda_2)^2(n+\beta+1)} \quad (\mbox{since } \alpha > 1/2 \mbox{ and } \sum_{i=1}^{n} i^{\gamma} \leq n^{\gamma+1}/(\gamma+1) \; \forall \; \gamma > -1). \label{eqn:main5}
	\end{align}
	Substituting~\eqref{eqn:main4},~\eqref{eqn:main4-5} and \eqref{eqn:main5} into~\eqref{eqn:master} proves the theorem.
\end{proof}

%% file: lemmas.tex
\section{Bounding the Convergence of Oja's Algorithm}
\label{sec:proofs}

In this section, we present a detailed proof of Theorem~\ref{thm:main1}. The proof follows the approach outlined in Section~\ref{sec:approach} and uses the notation of that section, i.e.
\begin{itemize}
	\item We let $\Bt{n} \eqdef \left(\eye + \eta_n \At{n}\right)\cdots \left(\eye + \eta_1 \At{1}\right)$ with $\Bt{0} \eqdef \eye$
	\item We let $\variancetotal \eqdef \variance+\lambda_1^2$
	\item We let $\Vperp \in \R^{d \times d - 1}$ denote a matrix whose columns form an orthonormal basis for the subspace orthogonal to $\v_1$ ~.
\end{itemize}
We first provide several technical lemmas bounding the expected behavior of $\B_n$ and ultimately use these lemmas to prove Theorem~\ref{thm:main1}. We begin with a straightforward lemma bounding the rate of increase of $\E{\B_t \B_t^\top}$ in spectral norm.

\begin{lemma}\label{lem:spectralboundE}
For all $t \geq 0$ and $\eta_i \geq 0$ we have
\[
\twonorm{\E{\Bt{t} \BtTr{t} }} \leq
\exp\left(\sum_{i\in[t]} 2\eta_i \lambda_1 + \eta_i^2 \variancetotal\right)
~.
\]
\end{lemma}

\begin{proof}
Let $\alpha_t \eqdef \|\E{\Bt{t} \Bt{t}^\top} \|_2$, i.e., $\E{\Bt{t} \Bt{t}^\top} \preceq \alpha_t \eye$. For all $t > 0$, 
\begin{align}
\E{\Bt{t} \Bt{t}^\top} 
&=
\E{(\eye + \eta_{t} \A_{t}) \Bt{t - 1} \Bt{t - 1}^\top (\eye + \eta_{t} \A_{t})^\top}
\preceq \alpha_{t-1} \E{ (\eye + \eta_{t} \A_{t}) (\eye + \eta_{t} \A_{t}^\top)},\notag
\\
&= \alpha_{t - 1} \E{\eye + \eta_t \A_t + \eta_t \A_t^\top + \eta_t^2 \A_t \A_t^\top}\preceq \alpha_{t - 1} \left[\eye + 2\eta_t \Cov + \eta_t^2 (\Cov^2 + V \eye) \right],\label{eq:sbE_1}
\end{align} 
where the last inequality follows from $\E{A_t}=\Cov$ and, 
\[
\E{\A_t\A_t^\top} = \Cov^2 + \E{(\A_t - \Cov)(\A_t - \Cov)^\top} \preceq \Cov^2 + V \eye
~.
\]
Using \eqref{eq:sbE_1} along with $\|\E{\Bt{t} \Bt{t}^\top}\|_2=\alpha_t$, $\Cov \preceq \lambda_1 \eye$, and $\Cov^2 \preceq \lambda_1^2 \eye$, we have  for $\forall t > 0$: 
\[
\alpha_{t} \leq (1 + 2\eta_t \lambda_1 + \eta_t^2 (\lambda_1^2 + V)) \alpha_{t-1}. 
\]
The result follows by using induction along with $\alpha_0=1$ and $1+ x \leq e^x$. 
\end{proof}

Using Lemma~\ref{lem:spectralboundE} we next bound the expected value of $\trace{\Vperp^\top \Bt{n} \BtTr{n} \Vperp}$. Ultimately this will allow us to bound the value $\trace{\Vperp^\top \Bt{n} \BtTr{n} \Vperp}$ with by Markov's inequality.

\begin{lemma}\label{lem:traceboundE}
For all $t \geq 0$ and $\eta_i \leq \frac{1}{\lambda_1}$ the following holds
\begin{multline*}
\E{\trace{\Vperp^\top \Bt{t} \BtTr{t} \Vperp}} \leq 
\exp\left(\sum_{j \in [t]} 2\eta_j \lambda_2 + \eta_j^2 \variancetotal \right) \cdot \left(d+\variance \sum_{i= 1}^{t}\eta_i^2 \exp
\left(\sum_{j\in[i]} 2\eta_j(\lambda_1-\lambda_2) \right)\right)
~.
\end{multline*}
\end{lemma}
%
\begin{proof}
Let $\alpha_t \eqdef \E{\trace{\Vperp^\top \Bt{t} \BtTr{t} \Vperp}}$. We first simplify $\alpha_t$ as follows:
\begin{align}
\alpha_t &= \iprod{\E{\Bt{t}\BtTr{t}}}{\Vperp \Vperp^\top }
= \iprod{\E{\Bt{t-1}\BtTr{t-1}}}{\E{\left(\eye+\eta_t\At{t}\right)
		\Vperp \Vperp^\top 
		\left(\eye+\eta_t\AtTrans{t}\right)} }.
\label{eqn:traceBoundE1}
\end{align}
Recall that $\E{\A_t}=\Cov$. Now, the second term on the right hand side can be bounded as follows:
\begin{align*}
&\E{\left(\eye+\eta_t\At{t}\right) \Vperp \Vperp^\top \left(\eye+\eta_t\AtTrans{t}\right)}, \\
&= \Vperp \Vperp^\top + \eta_t \Cov \Vperp \Vperp^\top + \eta_t \Vperp \Vperp^\top \Cov + \eta_t^2 \E{\At{t} \Vperp \Vperp^\top \AtTrans{t}}, \\
&= \Vperp \Vperp^\top + \eta_t \Cov \Vperp \Vperp^\top + \eta_t \Vperp \Vperp^\top \Cov + \eta_t^2 {\Cov \Vperp \Vperp^\top \Cov } + \eta_t^2 \E{\left(\At{t}-\Cov\right)\Vperp \Vperp^\top \trans{\left(\At{t}-\Cov\right)}}, \\
&\stackrel{\zeta_1}{\preceq} \Vperp \Vperp^\top + 2 \eta_t \lambda_2 \Vperp \Vperp^\top + \eta_t^2 \lambda_2^2 \Vperp \Vperp^\top+ \eta_t^2 \E{ \left(\At{t}-\Cov\right) \trans{\left(\At{t}-\Cov\right)} }, \\
&\stackrel{\zeta_2}{\preceq} \left(1+2 \eta_t \lambda_2 + \eta_t^2 \lambda_2^2 \right) \Vperp \Vperp^\top + \eta_t^2 \variance \eye = \left(1+2 \eta_t \lambda_2 + \eta_t^2 \lambda_2^2 + \eta_t^2 \variance \right) \Vperp \Vperp^\top + \eta_t^2 \variance \cdot \v_1 \v_1^\top,
\end{align*}
where $\zeta_1$ follows from the fact that $\Vperp$ is orthogonal to $\v_1$ and $\zeta_2$ follows from defintion of $\variance$. 

Plugging the above into~\eqref{eqn:traceBoundE1}, we get for all $t \geq 1$, 
\begin{align*}
\alpha_t &\leq \left(1+2 \eta_t \lambda_2 + \eta_t^2 (\lambda_2^2 + \variance)\right) \iprod{\E{\Bt{t-1}\BtTr{t-1}}}{\Vperp \Vperp^\top} + \eta_t^2 \variance \iprod{\E{\Bt{t-1}\BtTr{t-1}}}{\v_1 \v_1^\top}, \\
&\leq \left(1+2 \eta_t \lambda_2 + \eta_t^2 \variancetotal\right) \alpha_{t-1} + \eta_t^2 \variance \twonorm{\E{\Bt{t-1}\BtTr{t-1}}}, \\
&\leq \exp\left(2 \eta_t \lambda_{2} + \eta_t^2 \variancetotal \right) \alpha_{t-1} 
+ \eta_t^2 \variance \exp\left(
\sum_{i\in[t - 1]} \eta_i \lambda_1 + \eta_i^2 \variancetotal\right), 
\end{align*}
where the last inequality follows from $1+x\leq e^x$ and using Lemma~\ref{lem:spectralboundE}. 

Recursing the above inequality, we obtain
\begin{align*}
\alpha_t 
&\leq
\sum_{i \in [t]}
\eta_i^2 \variance
\exp \left( 
\sum_{j = i + 1}^{t}
2\eta_j \lambda_2 + \eta_j^2 \variancetotal\right)
\exp
\left(\sum_{j\in[i]} 2\eta_j \lambda_1 + \eta_j^2 \variancetotal \right)
+ \exp\left(\sum_{j \in [t]} 2\eta_j \lambda_2 + \eta_j^2 \variancetotal\right) \alpha_0,\\
& \leq\exp\left(\sum_{j \in [t]} 2\eta_j \lambda_2 + \eta_j^2 \variancetotal\right) \left(\alpha_0+\variance \sum_{i =1}^{t}\eta_i^2 \exp
\left(\sum_{j\in[i]} 2\eta_j(\lambda_1-\lambda_2) + \eta_j^2 \variancetotal \right)\right)
\end{align*}
Since $\Bt{0} = \eye$ we see that $\alpha_0 = d - 1 \leq d$. Using that $\eta_i \leq \frac{1}{\lambda_1} \leq \frac{1}{\lambda_2}$ completes the proof.
\end{proof}

Next we provide the lemmas that will allow us to lower bound $\v_1^\top \Bt{t} \Bt{t}^\top \v_1$. In Lemma~\ref{lem:lowerboundE} we lower bound $\E{\v_1^\top \Bt{t} \Bt{t}^\top \v_1}$ and in Lemma~\ref{lem:lowerboundV} we upper bound $\Var{\v_1^\top \Bt{t} \Bt{t}^\top \v_1}$. Ultimately, the lower bound follows using Chebyshev's inequality. 

\begin{lemma}\label{lem:lowerboundE}
For all $t \geq 0$ and $\eta_i \geq 0$ we have
\[
\E{\v_1^\top \Bt{t} \Bt{t}^\top \v_1}
\geq \exp \left(\sum_{i \in [t]} 2\eta_i \lambda_1 - 4\eta_i^2 \lambda_1^2\right)
\]
If we further assume that $\eta_i \leq \frac{1}{4\cdot \max\{\lambda_1,M\}}$ then
$
\E{\v_1^\top \Bt{t} \Bt{t}^\top \v_1}
\geq \exp (\lambda_1 \sum_{i \in [t]} \eta_i)
$.
\end{lemma}
\begin{proof}
	Let $\beta_t \eqdef \E{\v_1^\top \Bt{t}\BtTr{t} \v_1}$. Since $\Bt{t} = \left(\eye+\eta_t \At{t}\right) \Bt{t-1}$, we can bound $\beta_t$ as
	\begin{align*}
		\beta_t 
		&= \iprod{\E{\Bt{t-1}\BtTr{t-1}}}{\E{\left(\eye+\eta_t \At{t}\right) \v_1 \v_1^\top \left(\eye+\eta_t \AtTrans{t}\right)}} \\
		&= \iprod{\E{\Bt{t-1}\BtTr{t-1}}}{\v_1 \v_1^\top + \eta_t \Cov \v_1 \v_1^\top + \eta_t \v_1 \v_1^\top \Cov + \eta_t^2 \E{\At{t} \v_1 \v_1^\top \ustarTr\AtTrans{t}}}
		\\
		&\geq \iprod{\E{\Bt{t-1}\BtTr{t-1}}}{\v_1 \v_1^\top + \lambda_1 \eta_t \v_1 \v_1^\top + \lambda_1 \eta_t \v_1 \v_1^\top} ~.
	\end{align*}
Consequently $\beta_t \geq (1 + 2\eta_t \lambda_1) \beta_{t - 1}$. Furthermore, $\Bt{0}=\eye$ and hence $\beta_0 = \twonorm{\v_1}^2 = 1$. Proceeding by induction and using that  $1 + x \geq \exp(x - x^2)$ for all $x \geq 0$ finishes the proof.
\end{proof}

\begin{lemma}\label{lem:lowerboundV}
For $t \geq 0$ suppose that $\eta_i \leq \frac{1}{4\cdot \max\{\lambda_1,M\}}$ for all $i \in [t]$ then.
\[
\E{\left(\v_1^\top \Bt{t}\BtTr{t} \v_1\right)^2}
\leq 
\exp\left( \sum_{i \in [t]} 4\eta_i \lambda_1 + 10\eta_i^2 \variancetotal \right)
\]
\end{lemma}
\begin{proof}
	Let $\Wts{t}{s} \eqdef \left(\eye + \eta_t \At{t}\right)\cdots \left(\eye + \eta_{t-s+1}\At{t-s+1}\right)$ and $\gamma_s \eqdef \E{\left(\v_1^\top \Wts{t}{s}\WtsTr{t}{s} \v_1\right)^2}$. Note that $\Wts{t}{t}=\Bt{t}$ and $\gamma_t=\E{\v_1^\top\Bt{t}\BtTr{t}\v_1}$. Now, 
	\begin{align}
		\gamma_t
		&= \trace{\E{\WtsTr{t}{t}\v_1 \v_1^\top \Wts{t}{t} \WtsTr{t}{t}\v_1 \v_1^\top \Wts{t}{t}}} \nonumber \\
		&= \trace{\E{(\eye+\eta_1 \AtTrans{1}) \WtsTr{t}{t-1}\v_1 \v_1^\top \Wts{t}{t-1} (\eye+\eta_1 \At{1})(\eye+\eta_1 \AtTrans{1}) \WtsTr{t}{t-1} \v_1 \v_1^\top \Wts{t}{t-1} (\eye+\eta_1 \At{1})}} \nonumber \\
		&= \trace{\E{(\eye+\eta_1 \AtTrans{1}) \Gt{t-1} (\eye+\eta_1 \At{1}) (\eye+\eta_1 \AtTrans{1}) \Gt{t-1} (\eye+\eta_1 \At{1})}}, \label{eqn:upperboundV-1}
	\end{align}
	where  $ \Gt{t-1} \eqdef {\WtsTr{t}{t-1}\v_1 \v_1^\top \Wts{t}{t-1}}$. In order to bound the above quantity, we first bound the above expression for an arbitrary $\Gt{t-1} \equiv \G$. We then take an expectation over only $\At{1}$ and then finally take an expectation over $\Gt{t-1}$. That is, for an arbitrary fixed symmetric matrix $\G$, we have: 
	\begin{align}
		&\trace{\E{\left(\eye+\eta_1 \AtTrans{1}\right) \G \left(\eye+\eta_1 \At{1}\right)\left(\eye+\eta_1 \AtTrans{1}\right) \G \left(\eye+\eta_1 \At{1}\right)}} \nonumber \\
		&=\trace{\E{\left(\G + \eta_1 \At{1}^\top \G + \eta_1 \G \At{1} + \eta_1^2 \At{1}^\top \G \At{1}\right)^2}} \nonumber \\
		&= \mathrm{Tr}\left(\G^2 + \eta_1 \E{\AtTrans{1}} \G^2 + \eta_1 \G^2 \E{\At{1}} + \eta_1 \G\left(\E{ \At{1}} + \E{\AtTrans{1}}\right)\G \right. \nonumber \\
		& \quad + \eta_1^2 \E{\AtTrans{1} \G \At{1} \G} + \eta_1^2 \E{\AtTrans{1} \G \AtTrans{1} \G} + \eta_1^2 \E{\G \At{1} \G \At{1}} + \eta_1^2 \E{\G \AtTrans{1} \G \At{1}} \nonumber \\
		&\quad + \eta_1^2 \G \E{\At{1} \AtTrans{1}} \G 
		+ \eta_1^2 \E{\AtTrans{1} \G^2 \At{1}} + \eta_1^3 \E{\AtTrans{1} \G \left(\At{1}+\AtTrans{1}\right) \G \At{1}} 
		 \nonumber \\
		&\quad \left. + \eta_1^3 \E{\AtTrans{1} \G \At{1} \AtTrans{1} \G } + \eta_1^3 \E{ \G \At{1}\AtTrans{1} \G \At{1}} + \eta_1^4 \E{\AtTrans{1} \G \At{1}\AtTrans{1} \G \At{1}}   \right) \nonumber \\
		&= \trace{\G^2} + 4 \eta_1 \trace{\Cov \G^2} + 2 \eta_1^2 \trace{\E{\At{1} \AtTrans{1}} \G^2} + \eta_1^2 \trace{\E{\AtTrans{1} \G \At{1} \G}} \nonumber \\
		& \quad + \eta_1^2 \trace{\E{\AtTrans{1} \G \AtTrans{1} \G}} + \eta_1^2 \trace{\E{\G \At{1} \G \At{1}}} + \eta_1^2 \trace{\E{\G \AtTrans{1} \G \At{1}}} \nonumber \\
		&\quad + 2 \eta_1^3 \trace{\E{\AtTrans{1} \G \left(\At{1}+\AtTrans{1}\right) \G \At{1}}} + \eta_1^4 \trace{\E{\AtTrans{1} \G \At{1}\AtTrans{1} \G \At{1}}} \label{eqn:upperboundV-2}
	\end{align}
	We now bound the various terms above as follows. Each of the second order terms can be bounded using Lemma~\ref{lem:trace-cauchy-schwarz} as follows:
	\begin{align}
		\E{\trace{\AtTrans{1} \G \At{1} \G}} &\leq \frac{1}{2} \E{\frob{\AtTrans{1} \G}^2 + \frob{\At{1}\G}^2} \nonumber \\
		&= \frac{1}{2} \left(\trace{\G\E{\At{1}\AtTrans{1}}\G + \G\E{\AtTrans{1}\At{1}}\G} \right) \leq  (\variance + \lambda_1^2 ) \trace{\G^2}. \label{eqn:secondorder}
	\end{align} 
	The third order terms can be bounded as follows:
	\begin{align}
		\E{\trace{\AtTrans{1}\G\At{1}\G\At{1}}} &\leq \E{\twonorm{\At{1}} \trace{\AtTrans{1}\G \G\At{1}}} \nonumber \\
		&\leq \left(\upperbound + \lambda_1\right) \trace{\G\E{\At{1}\AtTrans{1}}\G}
		\leq \left(\upperbound + \lambda_1\right) \variancetotal \cdot
		 \trace{\G^2}. \label{eqn:thirdorder}
	\end{align}
	where we used the assumption that $\twonorm{\At{1}} \leq \twonorm{\At{1}-\Cov} + \twonorm{\Cov} \leq \upperbound + \lambda_1$ with probability $1$. Finally the fourth order term can be bounded as
	\begin{align}
		\trace{\E{\AtTrans{1} \G \At{1} \AtTrans{1} \G \At{1}}}
		&\leq \left(\upperbound + \lambda_1\right)^2 \trace{\G^2 \E{\At{1} \AtTrans{1}}}
		\leq \left(\upperbound + \lambda_1\right)^2 \variancetotal \cdot \trace{\G^2}. \label{eqn:fourthorder}
	\end{align}
	Plugging~\eqref{eqn:secondorder},~\eqref{eqn:thirdorder} and~\eqref{eqn:fourthorder} into~\eqref{eqn:upperboundV-2} tells us that
	\begin{align*}
		&\trace{\E{\left(\eye+\eta_1 \AtTrans{1}\right) \G \left(\eye+\eta_1 \At{1}\right)\left(\eye+\eta_1 \AtTrans{1}\right) \G \left(\eye+\eta_t \At{1}\right)}} \\
		&\quad \leq \trace{\G^2} + 4 \eta_1 \lambda_1 \trace{ \G^2} + 5 \eta_1^2 \variancetotal \cdot \trace{\G^2}  \\
		& \qquad + 4 \eta_1^3 \left(\upperbound + \lambda_1 \right) \variancetotal \cdot \trace{\G^2} + \eta_1^4 \left(\upperbound + \lambda_1 \right)^2 \variancetotal \cdot \trace{\G^2} \\
		&\quad = \left( 1 + 4\eta_1 \lambda_1 + 5\eta_1^2 \variancetotal + 4 \eta_1^3 \left(\upperbound + \lambda_1\right) \variancetotal + \eta_1^4 \left(\upperbound + \lambda_1\right)^2 \variancetotal \right) \trace{\G^2} \\
		&\quad \leq \exp\left( 4\eta_1 \lambda_1 + 10\eta_1^2 \variancetotal \right) \trace{\G^2}
	\end{align*}
	where in the last line we used that $\eta_i \leq \frac{1}{4 \max\{\upperbound,\lambda_1\}}$ and that $1 + x \leq \exp(x)$

	Using the value $\G = \Gt{t-1}={\WtsTr{t}{t-1}\v_1 \v_1^\top \Wts{t}{t-1}}$ and plugging the above into~\eqref{eqn:upperboundV-1}, we have
	\begin{align*}
		\gamma_t &= \trace{\E{\left(\eye+\eta_1 \AtTrans{1}\right) \Gt{t-1} \left(\eye+\eta_1 \At{1}\right) \left(\eye+\eta_1 \AtTrans{1}\right) \Gt{t-1} \left(\eye+\eta_1 \At{1}\right)}} \\
		&\leq \exp\left( 4\eta_1 \lambda_1 + 10\eta_1^2 \variancetotal  \right) \E{\trace{\Gt{t-1}^2}} 
		= \exp\left( 4\eta_1 \lambda_1 + 10\eta_1^2 \variancetotal \right) \gamma_{t-1},
	\end{align*}
	where we used the fact that $\gamma_{t-1} = \E{\trace{\Gt{t-1}^2}}$. Since $\gamma_0 = 1$, induction proves the lemma.
\end{proof}

We now have everything to prove Theorem~\ref{thm:main1}.

\begin{proof}[Proof of Theorem~\ref{thm:main1}]
	As discussed in Section~\ref{sec:approach} the main idea of this proof to use that Algorithm~\ref{algo:oja} is essentially one step of power method for the matrix $\Bt{n}$ and use Lemma~\ref{lem:onesteppowermethod} to bound the error.  To this end, we lower and upper bound $\v_1^\top\Bt{n}\BtTr{n}\v_1$ and	$\trace{\Vperp^\top\Bt{n}\BtTr{n}\Vperp}$, respectively.

First, using Chebyshev's inequality, we have: 
	\begin{align*}
		\P{\abs{\v_1^\top\Bt{n}\BtTr{n}\v_1 - \E{\v_1^\top\Bt{n}\BtTr{n}\v_1}} > \frac{1}{\sqrt{\delta}}\sqrt{\Var{\v_1^\top\Bt{n}\BtTr{n}\v_1}}} < \delta.
	\end{align*}
	So with probability greater than $1-\delta$, the following holds: 
	\begin{align}
		&\v_1^\top\Bt{n}\BtTr{n}\v_1 > \E{\v_1^\top\Bt{n}\BtTr{n}\v_1} - \frac{1}{\sqrt{\delta}}\sqrt{\Var{\v_1^\top\Bt{n}\BtTr{n}\v_1}} \notag\\
		&= \E{\v_1^\top\Bt{n}\BtTr{n}\v_1} - \frac{1}{\sqrt{\delta}}\sqrt{\E{\left(\v_1^\top\Bt{n}\BtTr{n}\v_1\right)^2} - \E{\v_1^\top\Bt{n}\BtTr{n}\v_1}^2}  \notag\\
		&\stackrel{\zeta_1}{\geq} \exp\left(2\lambda_1 \sum_{i=1}^{n}\eta_i -4\lambda_1^2\sum_{i=1}^{n}\eta_i^2\right) \times \left( 1 - \frac{1}{\sqrt{\delta}}\sqrt{ \exp\left(18\sum_{i=1}^{n} \eta_i^2  
			\variancetotal \right) -  1}\right) \notag \\
\label{eqn:main1}
		\end{align}
where $\zeta_1$ follows from Lemma~\ref{lem:lowerboundE} and \ref{lem:lowerboundV}. 

		Furthermore, using Lemma~\ref{lem:traceboundE} and Markov's inequality, we have with probability at least $ 1-\delta$,
		\begin{multline}
\trace{\Vperp^\top \Bt{t} \BtTr{t} \Vperp}
\leq \frac{\exp\left(\sum_{i \in [n]} 2\eta_i \lambda_2 + \eta_i^2 \variancetotal \right)}{\delta} \cdot \left(d+\variance \sum_{i=1}^{n}\eta_i^2 \exp
\left(\sum_{j\in[i]} 2\eta_j(\lambda_1-\lambda_2) \right)\right)
~. \label{eqn:main2}
		\end{multline}
		Consequently with probability at least $1 - 2\delta$ both  \eqref{eqn:main1} and \eqref{eqn:main2} hold and therefore the result follows by  Lemma~\ref{lem:onesteppowermethod} and choosing a $\delta$ that is smaller by a constant.
\end{proof}


%% file: conclusion.tex
\section{Conclusion and Future Work} \label{sec:conc}
This work presented a finite sample complexity and asymptotic convergence rates for the classic Oja's algorithm for top-$1$ component streaming PCA that match well known matrix concentration and perturbation results for computing the top eigenvector. In fact, asymptotically our bound improves upon standard matrix Bernstein bounds by a factor of $\order{\log d}$. Our results are tighter than existing streaming PCA results by a factor of either $\order{d}$ or $\order{1/\textrm{gap}}$.

Our analysis relied on a novel view of the algorithm and is technically fairly simple. We hope that our analysis opens a way to make progress on the many variants of PCA that occur in both theory and practice.
In particular, we believe the following directions should be of wide interest:
\begin{itemize}
	\item \textbf{Multiple components}: Currently, our result holds only for estimating the top eigenvector of $\Cov$. Extension of our technique to compute top-$k$ eigenvectors is an important future direction. 
	\item \textbf{Rayleigh quotient}: Another standard metric to measure optimality of $\wt{n}$ is Rayleigh quotient: $\wt{n}^\top \Cov \wt{n}$. Converting our bounds on $\sin^2(\wt{n},\v_1)$ to Rayleigh quotient loses a multiplicative factor of $\order{1/\textrm{gap}}$ compared to the optimal rate. A direct analysis that does not lose this factor is an interesting open problem. Results on Rayleigh quotient may also help in obtaining sample complexity guarantees that are independent of eigenvalue gap. 
	\item \textbf{High Probability}: This work focused on obtaining tight bounds on the error. However, the dependence of our results on success probability is quite suboptimal. One way to fix this is to run many copies of the algorithm, each with say $3/4$ success probability and then output the geometric median of the solutions, which can be done in nearly linear time\cite{cohen2016median}. However, we conjecture that a tighter analysis using our techniques might directly lead to improved dependency on success probability and possibly help solve some of the other problems mentioned above. 
\end{itemize}
